\documentclass[12pt]{extarticle}
\usepackage[utf8]{inputenc}

\usepackage{amsfonts,latexsym,amsthm,amssymb,amsmath,amscd,euscript, amsbsy}
\usepackage{framed}
\usepackage{fullpage}
\usepackage{color}
\usepackage{hyperref}
    \hypersetup{colorlinks=true,citecolor=blue,urlcolor =black}
\usepackage{verbatim}
\usepackage{tikz-cd}
\usepackage{wasysym}
\usepackage{enumitem}
\usepackage{mathtools}
\usepackage{comment}
\usepackage[ruled, vlined]{algorithm2e}
\usepackage{setspace}
\usepackage{float}
\usepackage{amsthm}
\usepackage{booktabs}
\usepackage{mdframed}

\allowdisplaybreaks[1]

\newtheorem{theorem}{Theorem}

\newtheorem{lemma}{Lemma}[section]

\theoremstyle{definition}

\theoremstyle{remark}

\newcommand{\T}{\mathbf{T}}
\newcommand{\bsigma}{\boldsymbol{\sigma}}

\newlist{thmcases}{enumerate}{1}
\setlist[thmcases]{
  label=\textbf{\upshape Case \arabic*:},
  leftmargin=*,
  ref={\thetheorem.\arabic*}}

\newlist{thmsteps}{enumerate}{1}
\setlist[thmsteps]{
  label=\textbf{\upshape Step \arabic*:},
  leftmargin=*,
  ref={\thetheorem.\arabic*}}
  
\newenvironment{textframe}%
  {\vspace{4mm}\begin{mdframed}[linewidth=1pt]}%
  {\end{mdframed}}

\title{Metrical Task Systems with Online Machine Learned Advice}
\author{Kevin Rao}
\date{September 2019}

\begin{document}

\maketitle

\begin{center}
    $\mathbf{Abstract}$
\end{center}

Machine learning algorithms are designed to make accurate predictions of the future based on existing data, while online algorithms seek to bound some performance measure (typically the competitive ratio) without knowledge of the future. Lykouris and Vassilvitskii \cite{LV} demonstrated that augmenting online algorithms with a machine learned predictor can provably decrease the competitive ratio under as long as the predictor is suitably accurate.

In this work we apply this idea to the Online Metrical Task System problem, which was put forth by Borodin, Linial, and Saks \cite{BLS} as a general model for dynamic systems processing tasks in an online fashion. We focus on the specific class of uniform task systems on $n$ tasks, for which the best deterministic algorithm is $O(n)$ competitive and the best randomized algorithm is $O(\log n)$ competitive.

By giving an online algorithms access to a machine learned oracle with absolute predictive error bounded above by $\eta_0$, we construct a $\Theta(\min(\sqrt{\eta_0}, \log n))$ competitive algorithm for the uniform case of the metrical task systems problem. We also give a $\Theta(\log \eta_0)$ lower bound on the competitive ratio of any randomized algorithm.

\newpage

\section{Introduction}

\subsection{Problem Statement}
In this work we consider the uniform metrical task system as an online problem, where online algorithms are equipped with a machine learning algorithm (or some other similar construct) that provides some advice which is meant to be helpful but may not be perfectly accurate.  

The \textit{metrical task system}, introduced by Borodin, Linial, and Saks in 1987 \cite{BLS}, models a number of online problems, including paging and k-servers. The goal of an online algorithm for task systems is, given a sequence of tasks on a system of $n$ processing states, to generate a schedule that completes the tasks with minimal cost. In their seminal paper, they prove that in the general case, no deterministic algorithm is better than $2n - 1$ competitive. In the special case of uniform task systems, where the cost of transitioning between any two states is fixed at 1, they prove an $\Omega(\log n)$ lower bound on the competitive ratio of any randomized algorithm, and provide an algorithm that is $H_n$-competitive, where $H_n$ is the $n$th harmonic number. 

In this paper, we present a deterministic algorithm for the uniform task systems problem with access to a machine learned oracle that has competitive ratio $\Theta(\sqrt{\eta_0})$, where $\eta_0$ is an upper bound on the predictive error of the machine learning oracle. We also prove that this algorithm is an optimal deterministic algorithm, and demonstrate how this algorithm can be mixed with a randomized strategy for a robust competitive ratio of $\Theta(\min(\sqrt{\eta_0}, \log n))$. Finally, we give a lower bound on the competitive ratio of randomized algorithms with machine learned advice for uniform task systems.

\subsection{Definitions}

We use the standard definition of online algorithms and their worst case analysis by modeling the problem as a game between a player and an adversary. We recommend referring to \cite{BLS} for specific definitions and \cite{Ravi} for examples. We refer to the optimal offline and clairvoyant algorithm as OPT.

A \textit{metrical task system (MTS)} is defined as a set of states $S$ and an associated transition cost $d(s_i, s_j)$ between any two states $s_i, s_j \in S$. Our work is only concerned with the specific case of \textit{uniform} task systems, which satisfy $d(s_i, s_j) = 1$ for all $i \neq j$ and $d(s_i, s_i) = 0$. We define $n \triangleq |S|$ and write a metrical task system as the ordered pair $(S, d)$. A \textit{task} $T$ is a length $n$ array where $T[i]$ is the cost of processing the task in state $s_i$. The input to an online scheduling algorithm is a sequence of tasks $\T$, one per time step, and we refer to the $i^{th}$ task of our sequence with the notation $\T^i$. In the worst case setting, an adversary chooses the number of tasks, the processing costs of all tasks, and the order that they are requested in the input.

A scheduling algorithm for a task system produces a \textit{schedule} in the form of a function $\bsigma: \mathbb{N} \rightarrow S$, where at time $t$ the algorithm is in state $\bsigma(t)$. By default, we say $\bsigma(0) = s_1$. The cost of a scheduling algorithm $A$ on an input $\T$ is the sum of state transition costs and task processing costs, or

$$c_A(\T) = \sum_{i = 1}^{|\T|} d(\bsigma_{A}(i - 1), \bsigma_{A}(i)) + \T^i[\bsigma_{A}(i)]$$

And in an \textit{online} scheduling algorithm, the state $\bsigma(i)$ is chosen based only on $\T^1, \T^2, \dots \T^i$ and $\bsigma(0), \bsigma(1), \dots \bsigma(i - 1)$.

We also use the definitions of the \textit{phase} of an algorithm for uniform task systems and \textit{saturation} of a state from BLS \cite{BLS}, which are as follows: Suppose phase $i$ of our algorithm begins at time $t_i$. At the beginning of a phase, all states are unsaturated. For each state $s$, call algorithm $A_s$ the algorithm that remains in state $s$ and never transitions to a different state. At time $t_i \leq t < t_{i + 1}$, we say a state $s$ is saturated at time $t$ if the cost that $A_s$ incurs on the tasks in the interval $[t_i, t]$ is greater than or equal to $1$. When all states are saturated, the next phase begins and all states are again unsaturated. We define a \textit{BLS phase algorithm} to be any algorithm that runs in phases, never processes tasks in a saturated state, and only leaves the current state when it becomes saturated. The intuition behind the last condition is that if we leave a state before it becomes saturated, we will need to revisit it before the end of the phase anyways, which is easily exploitable by an adversary. Notably, the only difference between different BLS phase algorithms is how they choose which unsaturated state to transition to when the current state becomes saturated.

We define a \textit{machine learned oracle} to be a black box that makes predictions about the future. We make no assumptions about the nature of the underlying machine learning model or the distribution of the error of the predictions, but we are given $\eta_0$, a worst case guarantee regarding the prediction error $\eta$ of our oracle. We call the advice from such an oracle \textit{Online Machine Learned Advice (OMLA)}, a term coined by Lykouris and Vassilvitskii in 2018\cite{LV}. In worst case analysis, an adversary can choose the oracle's predictions, subject to the condition that $\eta \leq \eta_0$. In our work we try to develop algorithms that have low competitive ratios for inputs with low prediction error, and for high error inputs don't perform worse than algorithms without an oracle.

\subsection{Previous Work}

\subsubsection{Metrical Task Systems}

In 1987, Borodin, Linial, and Saks propose the Oblivious BLS phase algorithm (shown below) for the uniform case of the metrical task system problem.

\begin{algorithm}[]
    \SetAlgoLined
    
    \While{There still remain unprocessed tasks}{
        Transition to a state uniformly at random
        
        \While{The current phase has not yet ended}{
            Process the next task request in the current state 
            
            \If{The current state is saturated}{
                \If{There are still unsaturated states}{
                    Transition to one of them uniformly at random
                }
                \Else{
                    End the current phase
                }
            }
        }
    }
 \caption{Oblivious BLS phase algorithm}
\end{algorithm}

This algorithm is described as ``oblivious" because its behavior only depends weakly on the task sequence. The analysis of the competitive ratio relies on the observation that when there are $k$ unsaturated states remaining, the probability that the algorithm is in the next state to be saturated is $\frac{1}{k}$. In particular, let $f(k)$ be the expected number of transitions the algorithm will make until the end of the phase given that there are $k$ unsaturated states. With probability $\frac{1}{k}$ the algorithm will make a state transition before there are $k - 1$ saturated states, so $f(k) = \frac{1}{k} + f(k - 1)$ and $f(n) = H_n$, the $n$th harmonic number.

\begin{lemma}
If in the worst case a BLS phase based algorithm makes $k$ state transitions per phase, it is $\Theta(k)$-competitive.
\label{lemma:trans_to_ratio}
\end{lemma}

\begin{proof}
Consider any BLS phase based algorithm $A$. Since our task system is uniform, $A$ incurs a cost of $1$ for each state transition, and since $A$ stays in a state only until it becomes saturated, $A$ incurs a task processing cost of at most $1$ in each state it visits during the phase. For $k$ state transitions $A$ processes tasks in at most $k + 1$ distinct states per phase (since it may not transfer out of its current state at the start of the phase). Thus, $A$ incurs a cost of at least $1$ and at most $2$ for each state that it transfers to, plus a possible additional $1$ if it processes tasks before its first state transition, for a total cost between $k$ and $2k + 1$.
 
During each phase, OPT either does or does not start out in the last state to be saturated. In the first case OPT can simply stay put the entire phase and incur a cost of 1, and in the latter case OPT can simply transition to the last state to be saturated and then stay put, for a cost of 2. Thus, OPT incurs a cost of at least $1$ and at most $2$ per phase, and $A$ is $\Theta(k)$-competitive.
\end{proof}

By lemma \ref{lemma:trans_to_ratio}, the competitive ratio of this oblivious algorithm is $\Theta(H_n)$, which is known to be $\Theta(\log n)$. Notably, lemma \ref{lemma:trans_to_ratio} equates an algorithm's competitive ratio and number of state transitions up to a constant factor of $2$, so for the rest of this work we simply concern ourselves with the latter.

Other algorithms that improved on BLS's work appeared in later literature. In 1990, Manasse, McGeoch, and Sleator demonstrate that their work on server problems can be applied to \textit{forced task systems} to give a $n - 1$ competitive algorithm for this restricted class of task systems \cite{MMS}. In 1998, Irani and Seiden prove that by using a different definition of phase and saturation, we can achieve a competitive ratio of $H_n + O(\sqrt{\log n})$ for uniform task systems. Since $H_n$ is the dominating term, this is only a marginal improvement BLS's algorithm. However, the framework they used is not as nice to work with; in particular, they transform the problem into a continuous time scheduling problem, and the running of their algorithm can't be partitioned into nicely independent phases.

The intuition behind why BLS's phase and saturation model performs well comes from the observation that an algorithm should avoid staying in a state where processing tasks is expensive, and favor states where processing tasks is cheap. Thus, the states that become saturated later tend to process the same tasks at a lower cost, and as a phase continues we tend to visit them. Furthermore, the phase and saturation model is particularly nice because it obviates any consideration of task processing costs, which is covered in the definition of saturation.

\subsubsection{Machine Learned Advice}

In recent years, the machine learned advice model has been studied alongside multiple online problems. Purohit, Svitkina, and Kumar worked on the non-clairvoyant job scheduling problem and introduced an oracle that predicts the processing time of each job \cite{Ravi}. They introduce an online algorithm with advice they call the Shortest Predicted Job First (SPJF) algorithm which performs well with low prediction error, then show that the SPJF algorithm can be combined with the well studied Round Robin (RR) algorithm for a Preferential Round Robin algorithm, which performs much better than RR with low prediction error while remaining no worse than RR with high prediction error. Lykouris and Vassilvitskii work on the online cacheing problem, which is known to be $\Omega(k)$ competitive in the deterministic case and has a $\Theta(\log k)$ competitive randomized algorithm \cite{LV}. By predicting the next time a page is requested, their oracle based modification of the classic Marker algorithm is $2 \cdot \min(1 + O(\sqrt{\eta / OPT}), 2H_k)$ competitive. Mitzenmacher takes a different approach to machine learned advice by describing a Sandwiched Learned Bloom Filter, which modifies the original data structure by including a learned function that attempts to predict the membership of any query key \cite{Mitz}. His construction allows the user to achieve much lower false positive rates as a trade off with the size of the learned function.

\section{Results}

We first discuss a model for machine learned advice that seems reasonable at a glance and has been used for related problems, but falls short for uniform task systems. This example illustrates some potential pitfalls of designing algorithms with online machine learned advice, and motivates the construction that yields the main results of this paper.

\subsection{A Motivating Discussion}

Following the example set by Lykouris and Vassilvitskii in their work on cacheing, define the LV oracle (after the authors) to be such that every time a specific task $T_i$ is requested, LV outputs $h(T_i)$, the predicted next time that task $T_i$ will be requested. We now prove by construction that against an adversary, any deterministic BLS phase algorithm with access to the LV oracle makes $n$ state transitions per phase and has linear competitive ratio:

\begin{theorem} Given any deterministic BLS phase algorithm $A$ with access to the LV oracle, an adversary can force $A$ to make $n - 1$ state transitions per phase. Furthermore, it can do so with $\eta = 0$ absolute predictive loss.
\label{thm:lv_bad}
\end{theorem}

Before getting into the proof of this theorem, we provide some intuition behind the strategy that the adversary chooses. Suppose we have any arbitrary length sequence of tasks $T_1, T_2, \dots $. The first key observation is that since $A$ is a deterministic algorithm, if the adversary is aware of how $A$ processed the first $i$ tasks, then the adversary is able to reliably predict how $A$ will process task $i + 1$. Thus, by induction the adversary is able to fully predict how $A$ will behave at any point on any sequence of task requests, and we use this fact to our advantage.

The second important observation is more subtle, and is discussed in \cite{BLS}. In essence, if the adversary chooses small task processing costs, then an online algorithm has less information each time it is faced with a decision. If instead of $k$ small tasks $T$ we have one large task $kT$, where the processing cost of $T$ in each state is scaled up by $k$, then one appearance of $kT$ would be equivalent to seeing $T$ and being given a ``guarantee" that the next $k - 1$ tasks are also $T$. Thus, our adversary chooses tasks with very small processing costs.

Now given that the adversary plans to precompute $A$'s behavior and use small small task sizes, all that's left is to devise an input that renders the information from a machine learned oracle as trivial as possible. Since the LV oracle predicts the next time each task appears, an adversary can obscure any ``useful" information by ensuring that every time $A$ needs to make a decision, each task will immediately be requested one time each. Thus, if there are $k$ tasks, $A$ can only make its decisions with knowledge of the next $k$ requests. Given how flexible the adversary can be, this reveals no useful information about each state's saturation time, which results in a high competitive ratio.

In summary, the issues with using the LV oracle for the uniform task system problem are

\begin{itemize}
    \item The adversary can set task processing costs as small as it likes, elongating the input and granting the adversary arbitrary flexibility.
    
    \item The competitive ratio of an algorithm for uniform task systems depends on the number of transitions per phase and order of saturation, which are only loosely related to the oracle's predictions of task reappearance time. In contrast, the LV oracle works well for the cacheing problem, where the only thing that matters is how much time elapses before a page is requested again. As a consequence,
    
    \item There exists an adversary strategy that works against even a zero error LV oracle by rendering the oracle's prediction more or less meaningless.
    
\end{itemize}

Now we give the formal proof of Theorem \ref{thm:lv_bad}:

\begin{proof} First the adversary sets the number of tasks to be $n$, picks any $m > n$, and defines the tasks to be:

\begin{align*}
    T_1 &= [\tfrac{1}{m}, 0, 0, 0 \dots 0] \\
    T_2 &= [0, \tfrac{1}{m}, 0, 0 \dots 0] \\
    \vdots \\
    T_n &= [0, 0, 0, 0 \dots \tfrac{1}{m}]
\end{align*}

That is, processing task $T_i$ incurs 0 cost in every state except $s_i$. The adversary alternates between two steps to generate the input:

\begin{thmsteps}
  \item Request each task exactly once.
  \item Since $A$ is deterministic, the adversary precomputes the next state $A$ transitions to, which we call $s_i$. The adversary then appends task $T_i$ to the input some $m'$ times, where state $s_i$ is saturated on exactly the $m'$th request. By the definition of BLS phase algorithm, upon fulfilling the $m'$th request, $A$ must now choose an unsaturated state to transition to, and the adversary returns to step $1$.
\end{thmsteps}

Following this process, at the beginning of a phase the adversary begins with:

$$T_1, T_2, \dots T_n, \dots$$

Now moving on to \textbf{Step 2}, the adversary determines $s_i$ \footnote{Here it's possible that $A$ is already in $s_i$ from the previous phase, in which case $A$ skips the first state transition of this phase.}, completes \textbf{Step 2}, and then performs \textbf{Step 1} again to continue the input:

$$T_1, T_2, \dots T_n, T_i^{m - 1}, T_1, T_2, \dots T_n, \dots$$

where $T_i^{m - 1}$ denotes requesting task $i$ exactly $m - 1$ times in a row. We observe that task $T_i$ has now been requested exactly $m$ times, so state $s_i$ becomes saturated right upon processing the last request for task $T_i$ and $A$ must make a transition. Notably, the oracle currently predicts that every task is about to be requested once, and is unable to ``see" past the next request for $T_n$. The adversary moves on to \textbf{Step 2} and again precomputes which state $A$ will transition to, say $s_j$, and continues the input as follows:

$$T_1, T_2, \dots T_n, T_i^{m - 1}, T_1, T_2, \dots T_n, T_j^{m - 2} \dots$$

And after performing \textbf{Step 1} again gets:

$$T_1, T_2, \dots T_n, T_i^{m - 1}, T_1, T_2, \dots T_n, T_j^{m - 2}, T_1, T_2, \dots T_n, \dots$$

and so on, alternating between \textbf{Step 1} and \textbf{Step 2} until the end of the phase. The adversary performs $\textbf{Step 1}$ exactly $n$ times, so each state's corresponding task has been requested at most $n$ times by the time $A$ chooses to transition to it. Since we chose $m$ so that $m > n$, and since each task must be requested $m$ times for its corresponding state to be saturated, we know no state will be saturated before the algorithm chooses to transition to it. Thus, $A$ visits each of our $n$ states for a total of at least $n - 1$ state transitions.

\end{proof}

By lemma \ref{lemma:trans_to_ratio}, $A$ is $\Theta(n)$-competitive. 

\subsection{Successful OMLA model}

With these issues in mind, we propose a different advice model that gives the algorithm more relevant and useful information. We give our online algorithms access to a Predicted Saturation Time (PST) oracle; a machine learned type oracle that for each phase attempts to predict the saturation time of each state with no ties (we assume that the PST oracle breaks ties randomly, and this makes very little difference in the results).

Specifically, at the beginning of each phase, the PST oracle outputs an array $h$, where $h[i]$ is the predicted saturation time of state $s_i$. The error $\eta$ of PST's prediction is defined to be the total $\ell_1$ error of our predictions. That is, if the true saturation time of state $s_i$ is $t_i$, then $\eta = \displaystyle \sum_{i = 1}^n |h[i] - t_i|$. Once again, the adversary controls both the PST oracle's predictions and the sequence of task requests, but is restricted to giving inputs with $\eta \leq \eta_0$ prediction error.

The Last Predicted State (LPS) algorithm queries the PST oracle at the beginning of each phase and always transitions to the unsaturated state with the latest predicted saturation time. Label the states $s_1, s_2, \dots s_n$, where $h[i] < h[i + 1]$. When the LPS algorithm's current state $s_i$ becomes saturated, suppose the $k$ unsaturated states remaining are $s_{i_1}, s_{i_2}, \dots s_{i_k}$, where $i_j < i_{j + 1}$. The LPS algorithm always transitions to state $s_{i_k}$. Importantly, we know that $i_k < i$, since when both $s_i$ and $s_{i_k}$ were unsaturated states, the LPS algorithm chose to transition to $s_i$. A nice corollary is the LPS algorithm transitions to states in strictly decreasing order by their label.

We also define the function

$Z(m) = \begin{cases} 
      \frac{m^2 - 1}{2} & m \equiv 1 \mod 2 \\
      \frac{m^2}{2} & m \equiv 0 \mod 2 \\
   \end{cases}$
   
and we define $Z^{-1}(\eta) = \lfloor \sqrt{2\eta + 1} \rfloor$, which has the property that for all $m$ and $Z(m) \leq \eta < Z(m + 1)$, $Z^{-1}(\eta) = m$.

\begin{theorem} Given a uniform metrical task system and $\eta_0$, the optimal adversary strategy for the LPS algorithm forces it to make $m = Z^{-1}(\eta_0)$ state transitions.
\label{theorem:adv_reverse}
\end{theorem}

\begin{proof} Suppose we label the states $s_1, s_2, \dots s_n$ in order of their predicted saturation time. To generate an input, the adversary must assign the predicted saturation times $h[1], h[2], \dots h[n]$ and the true saturation times $t_1, t_2, \dots t_i$. We first prove that setting the true saturation times of the last $m$ states in reverse is necessary to force the LPS algorithm to make $m$ state transitions. We then prove that the minimum error needed to force $m$ state transitions is $\eta = Z(m)$, which implies that with $\eta_0 < Z(m + 1)$ the adversary cannot force the LPS algorithm to make $m + 1$ state transitions.

Since the LPS algorithm always transitions to the unsaturated state with the latest predicted saturation time, it will always transition to state $s_n$ at the start of the phase. Thus, from the adversary's point of view, if we want to force the algorithm to make $m$ state transitions, we must have at least $m - 1$ states with true saturation time greater than $t_n$ (the true saturation time of state $s_n$), or else there won't even be enough unsaturated states left to eventually make $m$ state transitions in the phase.

Suppose that when $s_n$ is saturated the LPS algorithm transitions to some state $s_i$. We now need to force $m - 1$ more state transitions, and thus need at least $m - 2$ states with true saturation times greater than $t_i$. We can observe that by the definition of the LPS algorithm once again each of these $m - 2$ states have an earlier predicted saturation time than state $s_i$. By induction, the $m - 1$ states with true saturation time greater than $t_n$ must be saturated in decreasing order by predicted saturation time, so the LPS algorithm constantly transitions to the state with highest predicted saturation time which in reality is the next state to be saturated.

The $m - 1$ states we choose to be saturated after $s_n$ will end up having high true saturation times, so to minimize prediction error the best choice for them are the states that already have high predicted saturation times, namely  $s_{n - m + 1} \dots s_{n - 1}$. If (for convenience and without loss of generality) we let $t_1 = 1$, the adversary assigns predicted and true saturation times as follows:

\begin{figure}[H]
    \centering
    \includegraphics[width = 15cm]{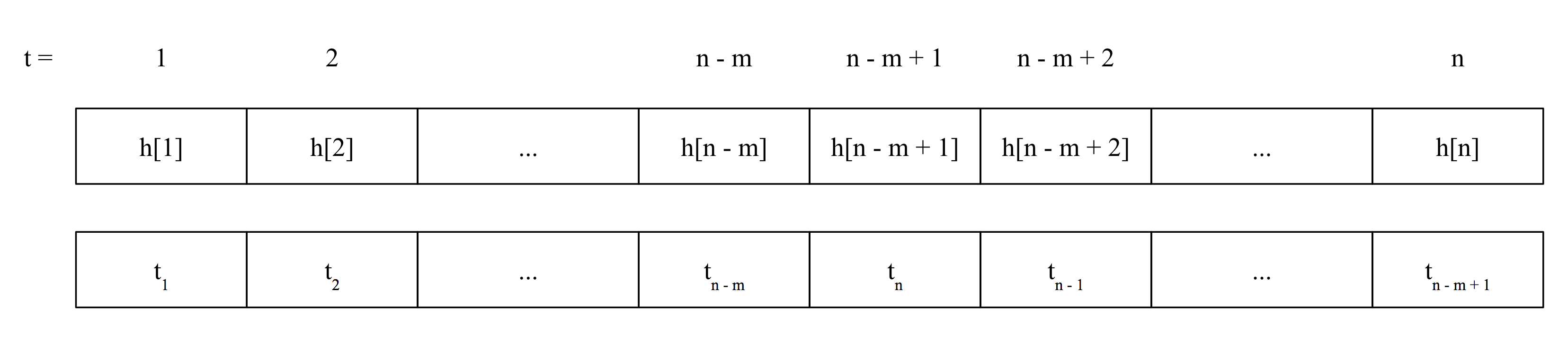}
    \caption{The true saturation times of the last $m$ states is reversed relative to their predicted saturation times. Since our analysis is confined to a single phase, without loss of generality we can assume the phase starts at $t = 1$.}
    \label{fig:adv_input}
\end{figure}

We now prove that this strategy indeed optimal, in the sense that it achieves the goal of forcing the LPS algorithm to make $m$ state transitions while minimizing the PST oracle's predictive error:

\begin{lemma} The optimal adversary strategy against the LPS algorithm assigns $t_i = h[i]$ for $i \leq n - m$ and $t_i = h[(n - i) + (n - m + 1)]$ for $n - m + 1 \leq i \leq n$.
\label{lemma:adv_assigns}
\end{lemma}

\begin{proof}

To have an easier time working with $\eta$ and the adversary's behavior, instead of thinking of $\eta_0$ as a restriction on the adversary, we think of predictive error as a resource that the adversary can use to force an online algorithm to behave in certain ways. Thus, given a certain task such as forcing the LPS algorithm to make $m$ state transitions, the adversary's strategy should try to do so while incurring as little predictive error as possible so it can be accomplished for as small of values of $\eta_0$ as possible.

First we observe that by the arguments made above, the $m$ state transitions will happen between states $s_{n - m + 1}, \dots s_n$. Since here the adversary is only concerned with forcing $m$ state transitions with as little error as possible, the adversary ``might as well" set $h[i] = t_i$ for states $s_1, s_2, \dots s_i, \dots s_{n - m}$ and incur 0 predictive error on them.

Next we show that given $h[n - m + 1] < h[n - m + 2] \dots < h[n]$ and $t_n < t_{n - 1} \dots < t_{n - m + 1}$, the optimal adversary strategy is to set $h[i] = h[i - 1] + 1$ and $t_i = t_{i - 1} - 1$. Suppose we have some state $s_k$ where $t_k < h[k]$ and $t_k < t_{k - 1} - 1$ for some positive integer $c$. Since there are no true saturation times between $t_k$ and $t_{k - 1}$, we can increase the true saturation time of $s_k$ by one (update $t_k$ to be $t_k + 1$) and decrease the total prediction error. Symmetrically, if $t_k > h[k]$ and $t_k > t_{k - 1} + 1$, if we decrease the true saturation time of $s_k$ by one (update $t_k$ to be $t_k - 1$) then we decrease the total prediction error.

Since predicted saturation times are in increasing order by state and true saturation times are in decreasing order, at most one state can have prediction error 0. Thus, the result of the above is that for every true saturation time assignment that is not contiguous, there exists an assignment with contiguous times that the LPS algorithm runs on with the same cost. A similar argument applies to show that the optimal adversary strategy also sets the predicted saturation times to be contiguous.

Finally, we prove that the optimal adversary strategy is to set $t_i = h[n - i + n - m + 1]$. We've established that the optimal adversary strategy involves assigning $h[i] = h[i - 1] + 1$ and $t_i = t_{i - 1} - 1$. We observe that if $t_{n - m + 1} < h[n - m + 1]$ or $t_n > h[n]$, we can always lower our total prediction error by translating either our predicted saturation times or true saturation times so that $t_{n - m + 1} = h[n - m + 1]$ or $t_n = h[n]$, respectively.

Now suppose $h[n - m + 1] - t_n = k$ for some integer $k$. It suffices to show that if $k = 0$ we minimize predictive error (Figure \ref{fig:k_opt}).

\begin{figure}[H]
    \centering
    \includegraphics[width = 15cm]{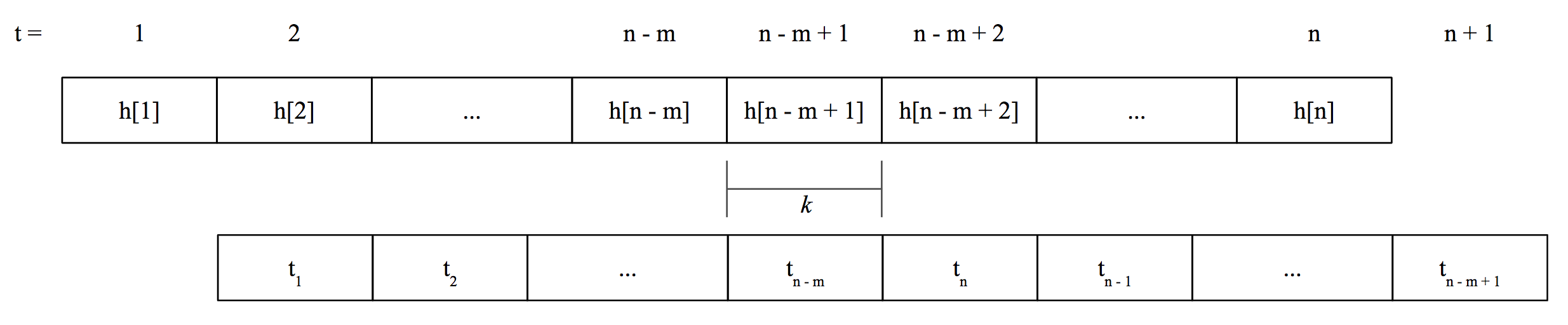}
    \caption{An example of the adversary's input when $k = 1$. Note that the predicted and true saturation times of $s_{n - m + 1}, \dots s_n$ are contiguous and that Figure \ref{fig:adv_input}, the optimal input, is identical except $k = 0$. Here $k$ is nonnegative, but the same argument applies for $k < 0$.}
    \label{fig:k_opt}
\end{figure}

It follows that $t_{n - m + 1} - h[n - m + 1] = m - 1 - k$, and $h[n] - t_n = m - 1 + k$. Suppose $m - 1 - k$ is even. The total prediction error is the sum of each state's prediction error, or 

\begin{align*}
    \eta &= (m - 1 - k) + (m - 3 - k) + \dots + 2 + 0 + 2 + \dots + (m - 1 + k) \\
    &= \frac{(m - 1 - k + 2)(\frac{m - 1 - k}{2})}{2} + \frac{(m - 1 + k + 2)(\frac{m - 1 + k}{2})}{2} \\
    &= \frac{(m + 1 - k)(m - 1 - k)}{4} + \frac{(m + 1 + k)(m - 1 + k)}{4} \\
    &= \frac{m^2 + k^2 - 1}{2}
\end{align*}

A similar calculation for odd $m - 1 - k$ gives $\eta = \frac{m^2 + k^2}{2}$. Thus, in either case the prediction error is minimized by picking $k = 0$, so the optimal adversary strategy is to set $t_i = h[n - i + n - m + 1]$, and this concludes the proof of lemma \ref{lemma:adv_assigns}.

\end{proof}

To complete the proof of Theorem \ref{theorem:adv_reverse}, we compute the prediction error $\eta = \displaystyle \sum_{i = 1}^n |h[i] - t_i|$ of the adversary strategy to show that $\eta \leq \eta_0$ as necessary. For odd $m$, we can see that 

\begin{align*}
    |h[i] - t_i| &= 0 \text{ for } i \leq n - m \\
    |h[n] - t_n| = |h[n - m + 1] - t_{n - m + 1}| &= m - 1\\
    |h[n - 1] - t_{n - 1}| = |h[n - m + 2] - t_{n - m + 2}| &= m - 3 \\
    &\vdots \\
    \left|h\left[\frac{2n - m + 1}{2}\right] - t_{\frac{2n - m + 1}{2}}\right| &= 0
\end{align*}

Thus, for odd $m$, the absolute prediction error is:

\begin{align*}
    \eta &= 2(m - 1) + 2(m - 3) + \dots + 2(2) + 1(0) \\
    \frac{\eta}{2} &= m - 1 + m - 3 \dots + 2 + 0 \\
    &= \frac{(m)(m + 1)}{2} - (1 + 3 + 5 + \dots + m) \\
    &= \frac{(m)(m + 1)}{2} - \left(\frac{m + 1}{2}\right)^2 \\
    &= \frac{m^2 - 1}{4} \\
    \eta &= \frac{m^2 - 1}{2} = Z(m)
\end{align*}

and a similar calculation gives $\eta = \frac{m^2}{2} = Z(m)$ for even $m$. Thus, if $\eta_0 \geq Z(m)$ the adversary can force the LPS algorithm to make $m$ state transitions. Since our adversary strategy forces $m$ state transitions with the lowest possible error, if $\eta_0 < Z(m + 1)$ the adversary cannot force $m + 1$ state transitions. Thus, the adversary strategy of reversing the last $m = Z^{-1}(\eta_0)$ states ordered by predicted saturation time is optimal against the LPS algorithm.

\end{proof}

By lemma \ref{lemma:trans_to_ratio}, the LPS algorithm is $\Theta(\sqrt{\eta_0})$ competitive. Notably, the LPS algorithm has competitive ratio exactly $1$ if $\eta_0 = 0$.

\begin{theorem}The LPS algorithm makes the fewest state transitions per phase out of all determinitic BLS phase algorithms when considering each algorithm's worst case input.
\end{theorem}

\begin{proof} Call the worst case number of transitions per phase of the LPS algorithm $m$, given by the adversary strategy above. We first show how an adversary would create an input that forces at least $m$ state transitions per phase for any deterministic BLS phase algorithm. We then show that the adversary can do so with the same $\eta_0$ as for the LPS algorithm, which implies that no deterministic BLS phase algorithm performs better than the LPS algorithm.

Suppose we have some deterministic BLS phase algorithm $A$. Again using the fact that the adversary can precompute the behavior of any deterministic algorithm, the adversary does as follows:

Assume without loss of generality that our phase starts at time $1$. First, set the predicted saturation times of the $n$ states to be $1, 2, \dots n$, and label the states $s_1, s_2, \dots s_n$ in order of predicted saturation time. Call the set $S_{> n - m} \subset S$ the set of states indexed greater than $n - m$. Next, set the true saturation times of states $s_1$ through $s_{n - m}$ to be the same as their predicted saturation times; that is, $t_i = i$ for $i \leq n - m$. The adversary also plans for the true saturation time of all states in $S_{> n - m}$ to be greater than $n - m$. Since BLS phase algorithms can only transition to unsaturated states, this guarantees that at time $n - m$, $A$ is in a state in $S_{>n - m}$ and that all states in $S_{> n - m}$ are currently still unsaturated.

Now, if we can guarantee that $A$ must transition to each state in $S_{>n - m}$ before the end of the phase, we will have constructed an input that forces $A$ to make at least $m$ transitions in this phase. To this end, let state $s$ be the first state in $S_{>n - m}$ that the algorithm transitions to. The adversary produces the rest of the input by making sure that the next state to be saturated is always the state that $A$ just transitioned to (and we prove this is possible in lemma \ref{lemma:adv_can_produce}). More formally, set the true saturation time of state $s$ to be time $n - m + 1$, and for $n - m + 1 \leq i < n$, set the true saturation time of state $\bsigma(i)$ to be $i + 1$. This means for all $m - 1$ values of $i$, at time $i + 1$ the algorithm transitions out of state $\bsigma(i)$ and into a new unsaturated state. Since the algorithm made some nonnegative number of state transitions before transitioning to state $s$, this gives an adversary strategy that can force any deterministic algorithm to make at least $m$ state transitions per phase. Thus, given the same max prediction error, an adversary can force any deterministic BLS phase algorithm to make at least as many state transitions as the LPS algorithm, which implies that we cannot construct an algorithm that does better than the LPS algorithm.
\end{proof}

\begin{lemma}
An adversary can saturate states $s_{n - m + 1} \dots s_n$ in any order with at most $Z(m)$ predictive error.
\label{lemma:adv_can_produce}
\end{lemma}

\begin{proof}
This is equivalent to proving that the Spearman footrule (Sf) distance between two lists of $m$ elements is at most $Z(m)$. As is standard with working with identity-invariant measures of permutation distance, we simply consider the distance of a permutation of $1, 2, \dots m$ to the identity.

We prove the bound on Sf distance by induction, and first observe that for $m = 1$ and $m = 2$ the bound holds. Now suppose we know that the Sf distance of a permutation on $m - 2$ elements is at most $Z(m - 2)$, and consider any permutation $\pi$ on $m$ elements. We perform two swaps on $\pi$ to get a permutation $\pi'$ with $Sf(\pi) \leq Sf(\pi')$, and the result follows from induction

Let element $m$ be in position $j$, and some element $i$ be in position $1$. If we swap the positions of element $m$ and $i$, the movement of $m$ increases the Sf distance by exactly $j - 1$, and the movement of $i$ decreases the Sf distance by at most $j - 1$. The same argument applies to swapping element $1$ into position $m$, giving us a permutation $\pi'$ with $1$ in position $m$, $m$ in position $1$, and elements $2$ through $m - 1$ forming a permutation on positions $2$ through $m - 1$.

\begin{figure}[H]
    \centering
    \includegraphics[width = 12cm]{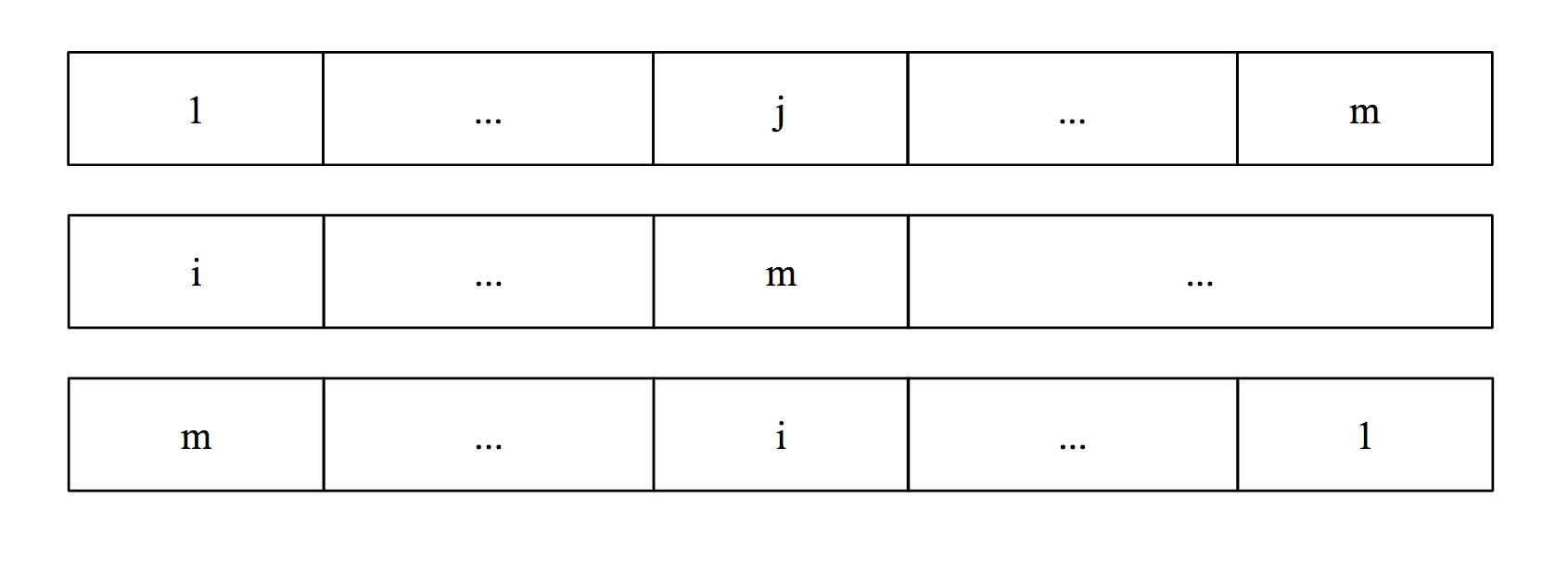}
    \caption{The identity (top), $\pi$ (middle), and $\pi'$ (bottom).}
    \label{fig:perm}
\end{figure}

This gives us $Sf(\pi) \leq Sf(\pi')$, and since elements $2$ through $m - 1$ are forming a permutation on $m - 2$ elements, our inductive hypothesis gives $Sf(\pi') = Z(m - 2) + 2(m - 1)$ which is equal to $Z(m)$ for both odd and even $m$. Thus, for any permutation $\pi$ on $m$ elements, $Sf(\pi) \leq Z(m)$.

\end{proof}

\subsection{Adjusting For Robustness}

The LPS algorithm as defined above relies on the quality of oracle's advice. If we imagine the oracle is very poorly trained and has up to $\eta_0 = n^2$ prediction error, our algorithm would run with competitive ratio $\Theta(n)$. Thus, we may wish to tweak our algorithm to perform well with low prediction error while still being robust to inputs with high prediction error. Here we make the standard assumption that the algorithm has no prior knowledge of $\eta_0$.

To maintain a low competitive ratio in the high error case, we have the LPS algorithm monitor its own progress. If it notices that it's making a lot of state transitions in the current phase, it switches to a known robust strategy. Let $k$ be some hyperparameter built into the algorithm that corresponds to this notion of ``making a lot of state transitions". The LPS algorithm runs normally in a phase until it makes $k$ state transitions, then switches to the BLS oblivious algorithm. There would be at most $n - k$ unsaturated states left, for a total expected $k + H_{n - k}$ state transitions. Since BLS's oblivious algorithm has competitive ratio $H_n$, we pick $k = H_n$.

Now we break the running of the LPS algorithm into two cases, based on whether the error $\eta_0$ is high enough to force $k = H_n$ state transitions and thus force the LPS algorithm to switch to the oblivious strategy.

\begin{textframe}

\noindent \textbf{Case 1:} $Z^{-1}(\eta_0) < H_n$, so it's impossible for the LPS algorithm to make $k = H_n$ or more state transitions per phase. In this case we never switch to the oblivious strategy, our previous analysis holds, and our algorithm is $\Theta(\sqrt{\eta_0})$ competitive.

\vspace{5mm}

\noindent \textbf{Case 2:} $Z^{-1}(\eta_0) > H_n$, so it's possible for the LPS algorithm to make $k = H_n$ or more state transitions per phase. In this case, we abandon our high-error predictor for the oblivious strategy after the $(H_n)$th state transition, leaving $n - H_n$ unsaturated states. Running the BLS oblivious algorithm on these $n - H_n$ states results in an expected $H_{n - H_n}$ state transitions, for a resulting competitive ratio of $\Theta(H_n + H_{n - H_n}) = \Theta(H_n) = \Theta(\log n)$ in the high error case.
\end{textframe}

This gives an overall competitive ratio of $\Theta(\min(\sqrt{\eta_0}, \log n))$.

\subsection{Comment on Randomized BLS Phase Algorithms}

Instead of examining any randomized algorithms, which we leave for future work, we instead prove the following bound on their performance:

\begin{theorem} No BLS phase randomized algorithm has competitive ratio lower than $H_m$ on uniform task systems, where $m = Z^{-1}(\eta_0)$.
\end{theorem}

\begin{proof} If $A$ is a randomized algorithm, we say the expected cost of $A$ on task sequence $\T$ is $\overline{c}_A(\mathbf{T}) = \sum_{\bsigma} c(\bsigma, \mathbf{T}) \Pr(\bsigma | \mathbf{T})$, where $\Pr(\bsigma | \mathbf{T})$ is the probability $A$ follows schedule $\bsigma$ given the input $\mathbf{T}$. $\overline{w}(A)$ is the infimum over all $w$ such that $A$ is $w$-competitive in expectation, and $\overline{w}(S, d)$ for a task system $S, d$ is the infimum of $\overline{w}(A)$ over all randomized algorithms $A$.

Suppose $D$ is some probability measure over the set of infinite task sequences, $\mathbf{T}^{\leq j}$ is the first $j$ task requests of infinite task sequence $\mathbf{T}$, and $E_D[c_0(\mathbf{T}^{\leq j})]$ tends to infinity with $j$. Call $m_j$ the minimum over all deterministic online algorithms $A$ of $E_D[c_A(\mathbf{T}^{\leq j})]$.

Define $D$ as follows: Pick any large integer $k$ and set the tasks to be

\begin{align*}
    T_1 &= [1/k, 0, 0, 0 \dots 0] \\
    T_2 &= [0, 1/k, 0, 0 \dots 0] \\
    \vdots \\
    T_n &= [0, 0, 0, 0 \dots 1/k]
\end{align*}

and based on $\eta_0$ compute $m = Z^{-1}(\eta_0)$. For each phase, we first make $k - 1$ requests to each of our $n$ tasks. We then concatenate the tasks $T_1, \dots T_{n - m}$ in order, uniformly choose a permutation of tasks $T_{n - m + 1}, \dots T_n$, and append it. We showed in lemma \ref{lemma:adv_can_produce} that this is possible with at most $\eta_0$ error.

Every phase of a BLS phase based algorithm would be exactly $kn$ time steps long, with each of the $m!$ possible phases appearing with equal probability and each state getting saturated when the corresponding task appears in the last $n$ task requests. The familiar harmonic number argument shows that any deterministic algorithm makes at least 
$H_m$ state transitions in expectation per phase. Adding in all task processing costs, each phase incurs cost at least $H_m$ as $k$ goes to infinity, so $m_j \geq H_m(\frac{j}{kn})$.

The optimal offline algorithm must incur a cost of $1$ per phase, so $E_D[c_0(\mathbf{T}^j)] = \frac{j}{kn}$. Thus, by Lemma 7.2 from \cite{BLS} which applies Yao's minimax principle,

$$\overline{w}(S, d) \geq \displaystyle \limsup_{j \rightarrow \infty} \frac{m_j}{E_D[c_0(\mathbf{T}^j)]} \geq \frac{H_m(\frac{j}{kn})}{\frac{j}{kn}} = H_m$$.

\end{proof}

$H_m$ is known to be $\Theta(\log m)$, so this result tells us that no randomized BLS phase algorithm is better than $\Theta(\log \sqrt{\eta_0}) = \Theta(\log \eta_0)$ competitive.

\section{Conclusion}

In this paper, we highlighted some of the challenges associated with augmenting online algorithms with a machine learning type oracle. We constructed the LPS algorithm, which performs well with good predictions and is robust to poorly trained oracles, and proved that it performs better than any other deterministic BLS phase algorithm. We also proved a lower bound on the competitive ratio of any randomized BLS phase algorithm.

\subsection{Future Work}

A few open questions still remain regarding the topics we address in this paper. The first is to investigate whether we can construct more competitive algorithms using Irani and Seiden's framework for solving task system problems. As we described before, their algorithms are much more complex than BLS's, but have slightly lower competitive ratios. Our paper also leaves the question of whether we can improve algorithms for general metrical task systems using machine learned advice. Finally, we leave the study of randomized online algorithms with advice for uniform task systems to future work.

\section*{Acknowledgments}

This work was done under the direction of Michael Mitzenmacher, Professor of Computer Science at the Harvard School of Engineering and Applied Sciences.

\end{document}